\documentclass[twoside]{article}

\usepackage{todonotes}
\usepackage[accepted]{twocolumn2023}
\usepackage{multirow}

\usepackage{amsmath,amsfonts,amsthm,amssymb, nccmath}

\usepackage{mathtools}
\usepackage[bookmarks=true,backref=page,colorlinks=true,allcolors=blue]{hyperref}

\usepackage{booktabs}
\usepackage{url}
\usepackage{algorithm2e}
\RestyleAlgo{ruled}
\usepackage{enumitem}

\usepackage{caption}
\usepackage{subcaption}

\usepackage{thmtools}
\usepackage{thm-restate}
\usepackage{makecell}

\usepackage{amsmath,amsfonts,bm}

\newcommand{\TT}{\mathcal T}

\def\eqref#1{equation~\ref{#1}}

\def\1{\bm{1}}

\DeclareMathAlphabet{\mathsfit}{\encodingdefault}{\sfdefault}{m}{sl}
\SetMathAlphabet{\mathsfit}{bold}{\encodingdefault}{\sfdefault}{bx}{n}

\newcommand{\R}{\mathbb{R}}

\newcommand{\OO}{\mathcal O}

\DeclarePairedDelimiterX{\dotp}[2]{\langle}{\rangle}{#1, #2}

\newcommand{\cb}[1]{{\color{red} (Cenk: #1)}}

\renewcommand{\cb}[1]{}
\newcommand{\embdim}{d_{\mathrm{emb}}}
\newcommand{\tablesize}{n}
\newcommand{\functions}{\mathcal{F}}
\newcommand{\lookup}{q}
\newcommand{\memtable}{T}
\newcommand{\cX}{\mathcal{X}}
\newcommand{\id}{\mathrm{id}}
\newcommand{\norm}[1]{\lVert #1 \rVert}
\newcommand{\N}{\mathbb{N}}
\newcommand{\experts}{\mathcal{T}}

\begin{document}

%

%
\runningauthor{The Power of External Memory in Increasing Predictive Model Capacity}

\twocolumn[

\aistatstitle{The Power of External Memory in Increasing Predictive Model Capacity}

\aistatsauthor{ Cenk Baykal \\
                Google Research \And 
                Dylan J Cutler \\
                Google Research \And  
                Nishanth Dikkala \\
                Google Research \AND 
                Nikhil Ghosh \\
                University of California, Berkeley \And 
                Rina Panigrahy \\
                Google Research 
                \And Xin Wang \\
                Google Research }



]

\begin{abstract}

One way of introducing sparsity into deep networks is by attaching an external table of parameters that is sparsely looked up at different layers of the network. By storing the bulk of the parameters in the external table, one can increase the capacity of the model without necessarily increasing the inference time. Two crucial questions in this setting are then: what is the lookup function for accessing the table and how are the contents of the table consumed? Prominent methods for accessing the table include 1) using words/wordpieces token-ids as table indices, 2) LSH hashing the token vector in each layer into a table of buckets, and 3) learnable softmax style routing to a table entry. The ways to consume the contents include adding/concatenating to input representation, and using the contents as expert networks that specialize to different inputs. In this work, we conduct rigorous experimental evaluations of existing ideas and their combinations. We also introduce a new method, alternating updates, that enables access to an increased token dimension without increasing the computation time, and demonstrate its effectiveness in language modeling.
\end{abstract}
\section{Introduction}
\label{sec:introduction}

Contemporary machine learning models have been remarkably successful in many different domains ranging from natural language~\cite{chowdhery2022palm,hoffmann2022training} to computer vision~\cite{yu2022coca,riquelme2021scaling}. However, these successes have come in part through sheer scale. A vast amount of empirical studies justify the conventional wisdom that bigger (models and data sets) is better~\cite{hernandez2021scaling,kaplan2020scaling}. Accordingly, state-of-the-art models often contain billions of parameters and are trained for weeks on enormously large data sets using thousands of AI accelerators. Their immense size leads to prohibitive compute and energy costs~\cite{patterson2021carbon} and prevents their deployment to resource or compute-constrained applications (e.g., autonomous driving)~\cite{liebenwein2021lost}.

\emph{Sparsely-activated networks}, such as Mixture-of-Expert (MoE) models~\cite{shazeer2017outrageously}, have the potential to alleviate these costs and enable efficient scalability of modern models. The main idea is to partition a network's or each layer's parameters into a table (of experts), where each entry (expert) of the table corresponds to a small subset of disjoint parameters that can be acted on by the input. During training and inference, a given input to the network is \emph{routed} to a small subset of entries (parameters) to compute the output. As a result, the computation cost remains small relative to the total number of network parameters. By storing the bulk of the parameters in externally accessed tables, we obtain models with significantly higher capacity with only a relatively small increase in computation time.

Designing effective sparsely-activated models hinges on two essential components: (i) the expert lookup (routing) function and (ii) the logic for consuming the contents of the table entries. Examples of lookup functions include using Token-ID lookups similar to~\cite{roller2021hash}, Locality Sensitive Hashing (LSH) lookup of input token embeddings~\cite{panigrahy2021sketch}, and trainable softmax based lookup as in sparse expert models~\cite{fedus2022review,lepikhin2020gshard}. There are also different ways of consuming the accessed entries. For example, one could view the accessed entries as additional parameters for input representation, which can be added or concatenated with the layer's output to form the augmented output. Alternatively, one could interpret the table entry as input-dependent function parameters, which parameterize an \emph{expert} function whose output is combined with the \emph{main} expert output;  here, the \emph{main} expert is one that the input is always acted upon (see Fig.~\ref{fig:moe}). Overall, there exists a research gap in evaluating and comparing combinations of such ideas to generate the most efficient and high-performing sparse models.



\begin{figure}[ht]
  \centering
  \begin{subfigure}[b]{0.35\linewidth}
    \centering
    \includegraphics[width=\textwidth]{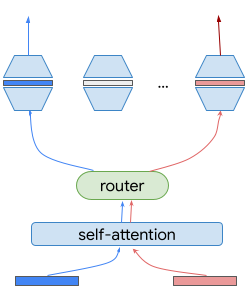}
    \caption{Mixture of Experts }
  \end{subfigure}
  \begin{subfigure}[b]{0.62\linewidth}
    \centering
    \includegraphics[width=\textwidth]{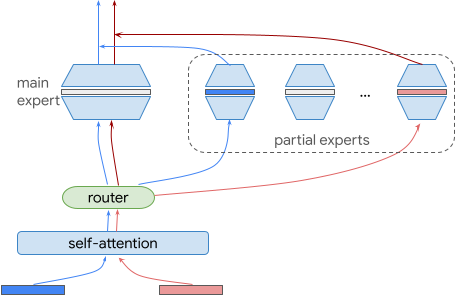}
    \caption{Mixture of Partial Experts}
  \end{subfigure}
  \caption{The standard Mixture of Experts model (left) routes the inputs to one or more of $n$ experts based on a routing function. Mixture of Partial Experts (right) always routes the input to the main expert and additionally routes the input to one or more partial experts; the output is a function of the main expert's and partial experts' outputs.}
   \label{fig:moe}
\end{figure}
In this work, we extensively evaluate several choices for the lookup function and the consumption logic and their combinations. We find Token-ID lookup to be more effective in the large-number-of-experts case.  From a theoretical perspective, we provide insights into popular lookup (routing) functions and why some might perform better than others. In addition, inspired by the observation that transformer models \cite{vaswani2017attention} benefit from increased representation dimension, we use memory parameters as additional parameters for input representation, and introduce a novel method called \emph{Alternating Updates}. This method widens the representation without increasing transformer computation time by working on a part of the representation at each layer. In particular, our contributions are:
\begin{enumerate}
    \item  We extensively study and empirically evaluate various lookup functions including softmax, LSH, and Token-ID lookup in the partial experts setting.
    
    \item We show the theoretical connections between various lookup functions and LSH variants. We also establish the power of consuming disjoint embedding tables at latter layers of the network.
    
    
    \item We introduce and evaluate the method of \emph{Alternating Updates} that enables increased token dimension with little additional computation cost.

\end{enumerate}

\section{Related Work}
\label{sec:related-work}

Prior work is rich with a diverse set of techniques to increase the efficiency of contemporary transformer models (see~\cite{tay2020efficient} for a survey). In this paper, we focus on lookup-based sparse models due to their state-of-the-art performance on various standard benchmarks~\cite{fedus2022review} and favorable theoretical properties~\cite{chen2022towards,baykal2022theoretical}.

Recent works have introduced extremely large, yet scalable models with the use of conditional routing of inputs to a learnable subset of parameters. Notably, the Sparse Mixture of Experts (SMoE)~\cite{shazeer2017outrageously,yuksel2012twenty,jacobs1991adaptive} family of models use a learned softmax probability distribution to conditionally direct the computation to \emph{experts}, i.e., subsets of network parameters. By routing the computation to a small subset of parameters on an input-dependent basis, SMoE leads to higher capacity models with a relatively small and controllable increase in computation. Switch Transformers~\cite{fedus2021switch} show that routing to a single expert on an input-dependent basis reduces computation and outperforms prior SMoE approaches on language tasks. 

Follow up work on SMoE include those that improve the load balancing of experts~\cite{zoph2022designing,lewis2021base}, use reinforcement learning to learn the routing function~\cite{clark2022unified}, and leverage smooth top-$k$ expert selection~\cite{hazimeh2021dselect}
(see~\cite{fedus2022review} for a survey). Other choices for the routing function include non-learnable ones such as Locality Sensitivity Hashing (LSH)~\cite{panigrahy2021sketch} which generally maps similar inputs to the same expert, Hash Layers that use token-based hashing~\cite{roller2021hash}, and language-specific deterministic routing~\cite{fan2021beyond}. Residual Mixture of Experts~\cite{wu2022residual} separates the expert weights into input-independent and input-dependent components, similar to the partial expert setting we have.

Conditionally accessing \emph{external memory} is another related approach to vastly increase model capacity at the cost of a relatively small increase in computation~\cite{graves2016hybrid,graves2014neural}. For examples, Memorizing Transformers~\cite{wu2022memorizing}, Memformer~\cite{wu2020memformer}, and Product key memory~\cite{lample2019large} leverage dynamic memory to encode and retrieve relevant information. Additional works include those that use an immensely large untrainable corpus, such as Wikipedia, REALM~\cite{guu2020retrieval}, or a 2 trillion token database, RETRO~\cite{borgeaud2022improving}.





\section{Lookup Functions}
\label{sec:routing-functions}

In this section, we formalize the augmentation of a layer $L$ with external memory and provide an overview of the various lookup functions that we cover in this paper. The memory augmented layer is shown as Alg.~\ref{alg:mal}. Since we are primarily interested in Transformer-like architectures we can just focus on the action of the layer on a single token. Abstractly, we consider a layer $L$ to be a function from a vector space $\cX$ to itself. For example, $L$ can be the self-attention layer of a transformer and $\cX$ the space of embedded token vectors $\R^{\embdim}$. We let $\functions$ be the set of all functions from $\cX$ to $\cX$. The external memory for the layer consists of a lookup function $\lookup$ and a memory table $\memtable$.

Given the previous layer output $x \in \cX$ and the index of original token in the vocabulary (i.e.\ the token-id, see Sec.~\ref{sec:token-id-routing}) which we denote $\id$, the look-up function $\lookup$ computes a set of indices $\experts = \lookup(x, \id)$. Typically the look-up function either uses only $\id$ as in Token-ID lookup (see Sec.~\ref{sec:token-id-routing}) or only $x$ as in Softmax lookup (see Sec.~\ref{sec:softmax-routing}). The indices $i \in \experts$ are then used to access the table $\memtable$ and access experts $f_i = \memtable(i)$ in $\functions$. 

We will consider experts $f_i$ that are either two-layer fully connected networks $f(x) = V\phi(U^Tx)$ where $\phi$ is the ReLU activation or just simply a constant function $f(x) = b$. For a $d$-dimensional input, the matrices $V,U \in \mathbb{R}^{d \times \mathrm{rank}}$, where  $\mathrm{rank}$ is a configurable parameter that specifies the width of the expert, and consequently the computation time of routing to each partial expert. The augmented layer computation outputs $L(x) + \sum_{i \in \experts} w_i(x) f_i(x)$ for some weighting functions $w_i$ instead of the normal output $L(x)$. In principle, one could consider alternate ways of combing $L(x)$ and $\sum_{i} w_i(x) f_i(x)$, however in this work we only consider addition since it is simple and preserves dimensions. We consider different choices of the lookup function $\lookup$ and memory tables $\memtable$ as follows.
\begin{algorithm}
\label{alg:augment-memory}
\SetAlgoLined
\KwIn{Layer $L \in \functions$, previous layer output $x \in \cX$, $\id \in \N$, look-up function $\lookup : \cX \times \N \to [\tablesize]$, memory table $\memtable : [\tablesize] \to \functions$
}
Table index $i = \lookup(x, \id)$\;
Adjustment function $f = \memtable(i)$\;
\KwOut{$L(x) + f(x)$}
\caption{Memory Augmented Layer}\label{alg:mal}
\end{algorithm}

\subsection{MoE-style Softmax Lookup}\label{sec:softmax-routing}
The MoE layer routes an input token $x$ to $k$ of $n$ experts where each expert is itself a parametrized subnetwork (e.g., a fully-connected layer). Following~\cite{fedus2022review}, we let $\{E_i(\cdot)\}_{i \in [n]}$ and $E_i(x)$ denote the set of experts and the output of lookup the input token $x$ to expert $i$, respectively. For an input token $x$, a learnable weight matrix $W$ is applied to obtain the logits $h(x) = W x$. The lookup probabilities are computed by taking the softmax of $h(x)$
$$
p_i(x) = \frac{\exp(h_i(x))}{\sum_{j \in [n]} \exp(h_j(x))} \quad \forall{i \in [n]}.
$$
The token $x$ is routed to the expert(s) $\TT \subset [n]$ with the top-$k$ probabilities $p(x)$. Since this operation is not differentiable, the output $y$ is computed as a probability weighted combination of the experts' outputs to enable gradients to propagate back to the router parameters, i.e.,
$
y = \sum_{i \in \TT} p_i(x) E_i(x)
$
~\cite{shazeer2017outrageously}.

\subsection{Token-ID Lookup}\label{sec:token-id-routing}
In a transformer, the computation can be viewed as repeatedly transforming the embedding vector of a token within the initial embedding space. For example the token ``tiger" may be embedded initially as $x_0 \in \R^{\embdim}$ and then transformed into $x_1, x_2, \ldots$, etc.\ successively where each $x_i \in \R^{\embdim}$. In Token-ID lookup for each $x_i$ the look-up function $\lookup$ (see Alg.~\ref{alg:mal}) simply returns the index of the input token (e.g.\ ``tiger") in the vocabulary and ignores the previous layer output. Note that in this case the table size $\tablesize$ of each layer is equal to the size of the vocabulary.

\subsection{Locality Sensitive Hashing (LSH) Lookup}
Locality Sensitive Hashing~\cite{gionis1999similarity} is a popular variant of hashing that tends to hash similar objects to the same buckets and is used for approximate nearest neighbor search~\cite{andoni2014beyond, andoni2015optimal, andoni2015practical}. It tends to hash similar inputs to the same bucket with higher probability and dissimilar inputs to different buckets. There are variants of LSH, but for LSH lookup we follow prior work~\cite{baykal2022theoretical,panigrahy2021sketch} and consider the \emph{hyperplane-based LSH}. At a high level, this approach partitions the input space into a grid-like structure using randomly oriented, equispaced hyperplanes. Each such region is considered a bucket (expert) of the hash table. See Sec.~\ref{sec:analysis} and the supplementary material for details.

\section{Theoretical Arguments}
\label{sec:analysis}
In this section, we analyze and provide unifying theoretical insights into popular lookup functions (Sec.~\ref{sec:analysis-lookups}) and demonstrate the theoretical advantage of using embedding table lookups at higher layers (Sec.~\ref{sec:analysis-embedding-lookups}).

\subsection{Lookup functions as variants of LSH and their efficiency}
\label{sec:analysis-lookups}
Here, we show that under simplifying assumptions, softmax routing and wordpiece routing can be viewed as Spherical LSH and Min-hash LSH, respectively. This interpretation will imply that, with practical configurations, Token-ID is more parameter efficient than Softmax, which is more efficient than hyperplane-LSH lookup.

{\bf Preliminaries}  We consider a Locality Sensitive Hashing (LSH) that maps an input to one of $n$ buckets. We let $r_2 > r_1 > 0$, denote the threshold for nearby points and far-away points, respectively. For $x, y \in \mathbb{R}^d$, we say $x$ and $y$ are nearby if $
\norm{x - y}_2 \leq r_1$ and they are far-away if $\norm{x - y}_2 \geq r_2$, where $\norm{x}_2$ is the $2$-norm of the vector $x$. Let $c = r_2/r_1 > 1$ denote the distance gap as a ratio. Let 
\begin{align*}
    p_1 &\le \Pr(h(x) = h(y): \norm{x - y}_2 \leq r_1)\\
    p_2 &\ge \Pr(h(x) = h(y): \norm{x - y}_2 \geq r_2)
\end{align*}
 denote lower and upper bounds on the collision probability of nearby points and far-away points, respectively. Notably, with $n$ buckets the probability that two nearby points hash to the same bucket is $n^{-\rho}$, where $\rho = \frac{\log (1/p_1)}{\log (1/p_2)}$~\cite{andoni2015optimal}.

{\bf Efficiency} Let us consider two sentences $s_1, s_2$ of the same length $l$ that have $f$ fraction of wordpieces in common. Assume for simplicity that the embedding vector for each wordpiece is a random unit vector in $R^d$. We summarize LSH variants and the collision probability of nearby points in this setting. We are interested in the \emph{efficiency}, i.e., the collision probability for the set of experts corresponding to two similar sentences. For a fixed size table of $n$ experts, the higher the collision probability for two similar sentences, the more efficient the LSH lookup is in terms of routing similar tokens to similar buckets. The full details of the LSH variants and proofs are in the supplementary.
\begin{enumerate}
    \item \textbf{Hyperplane LSH~\cite{datar2004locality}}: this variant divides a $R^d$ space into buckets by using randomly oriented parallel equispaced hyperplanes. Hyperplane LSH has the property that $\rho = \OO(1/c)$. Computations (see supplementary) yield $c = 1/\sqrt{1- f}$, which implies a collision probability of $n^{-\OO(\sqrt{1- f})}$.
    
    
    
    \item \textbf{Spherical LSH \cite{andoni2008near}}: here, we use a random set of points to divide up the $R^d$ space into Voronoi regions, each representing a different bucket. This method has a better $\rho$ value of $\OO(1/c^2)$. Assuming that the Softmax lookup matrix $W$ (see Sec.~\ref{sec:routing-functions}) is uniform, the Softmax lookup corresponds to Spherical LSH~\cite{andoni2015practical}. This yields $\rho = \OO(1- f)$ and a collision probability of $n^{-\OO(1- f)}$.
    
    \item \textbf{Min-hash~\cite{broder1998min}}: this approach is used for hashing sets so that similar sets get hashed to the same bucket. Using the Jaccard similarity measure for comparing sets means that the fraction of experts that match up for the two sentences $s_1,s_2$ is $f$. This also means that Token-ID lookup can be viewed as Min-hash LSH.
    
    
\end{enumerate}

The above properties imply the following theorem.

\begin{restatable}{theorem}{lookupefficiency}
\label{thm:lookup-efficiency}
In the setting of the above simplifying assumptions, we have the following:
\begin{enumerate}
    \item Softmax and Token-ID lookup can be viewed as Spherical LSH and Min-hash LSH, respectively.
    \item The probability that a random token in the two sentences of equal length that overlap in $f$ fraction of the wordpieces gets routed to the same expert is $n^{-\OO\left(\sqrt{1- f}\right)}$, $n^{-\OO\left(1- f\right)}$, and $f$ for hyperplane LSH, Softmax, and Token-ID lookup, respectively.
    \item For large $n$ and a small fraction $f$, in terms of routing to the same expert, the efficacy order of the different lookup methods is $\text{Token-ID} \ge \text{Softmax} \ge \text{hyperlane LSH}$. 
\end{enumerate}
\end{restatable}



\subsection{Advantage of embedding lookups at higher layers}
\label{sec:analysis-embedding-lookups}

Note that the Token-ID routing need not be a per-layer routing operation since it is merely a function of the token ID which can be done once in the input layer. One way of incorporating the result of this lookup is to simply feed the output of the lookup into the input for the first layer. This implementation delegates the work of passing on this information to the higher layers to the network itself. Alternatively, the output of the lookup in the input layer can be partitioned so that \emph{different parts of the lookup output feed into the different layers of the network}. This partitioning and feeding the embedding lookup directly to the layers can be viewed as separate embedding lookups in those layers. 

The theorem below establishes that the latter implementation with embedding lookups at higher layers enables a more efficient architecture of lower width and fewer parameters than the former one.


\begin{restatable}{theorem}{embeddinglookup}
\label{thm:embedding-lookup}
There exists a class of natural learning problems where embedding lookups of categorical features at upper layers in addition to the input layer gives a more efficient architecture compared to an architecture that feeds the embedding lookup output only to the input layer.
\end{restatable}
\begin{proof}[Proof Sketch]
The main idea is to consider two architectures that implement the embedding lookup in the two distinct ways and an input $(u, q)$ with ground truth score $\dotp{\Psi(u)}{\Phi(q)}$, where $\Psi(u)$ maps $u$ to a $d$-dimensional feature vector and $\Phi(q)$ is a non-linear transformation of $q$ that can be implemented by a deep network of width $d$. The first architecture combines the lookup $\Psi(u)$ with $q$ (by a weighted sum) and feeds into the network as input; the second architecture in addition feeds the embedding output of $u$ to all the layers of the network instead of only the lowest layer. The second architecture can store $\Psi(u)$ in the table and feed it directly to the output layer (which produces $\Phi(q)$) to obtain the result $\dotp{\Psi(u)}{\Phi(q)}$ using width $d$. On the other hand, for the first architecture the entropy of the information carried up the layers is at least $2d$ assuming $u$ and $q$ are random and not correlated, and so the width of the network needs to be $2d$.

\end{proof}
\section{External memory with Alternating Updates}
\label{sec:alternating-updates}
In this section, we introduce the method of \emph{Alternating Updates}, an approach to enable increased token dimension with little additional computation cost. 

\subsection{Background}
Instead of viewing external memory parameters as input-dependent function parameters (or ``experts”), we can view them as additional parameters for the input representation. To consume these additional parameters, we can project and add them to the original representation vector. Alternatively, we can use them to widen the representation vector, as we do in Alternating Updates. This ties well with the observation that language models benefit from wider model dimensions, for example, as model sizes scale up, model dimension grows from 512 (small) to 768 (base) and 1024 (large, 3B, and 11B) in T5 models \cite{raffel2020exploring}, and from 4096 (8B) to 8192 (64B) and 18432 (540B) in PaLM models \cite{chowdhery2022palm}.

As the model dimension increases, both representation dimension and transformer layer dimension increase. However, the two dimensions account for different capacities of the model: wider representations store more information about the input, while wider transformer layers give more processing power. They also differ a lot in computation cost: widening the representation increases computation minimally, while widening transformer layers quadratically increases computation cost. A natural question is how to incorporate wider representations while maintaining smaller transformer layers.

\subsection{A Predict-Compute-Correct Algorithm}
  We propose to keep a wide representation vector, perform computation with a sub-block, and estimate the updated representation using a Predict-Compute-Correct algorithm, as illustrated in Figure~\ref{fig:pcc}. Taking this view, external memory is now added by increasing the token embedding's dimensionality: suppose the original embedding dimension is $d \in \mathbb{N}$, it is now increased to $d + e \in \mathbb{N}$, which introduces $V e$ additional parameters, where $V$ is the vocabulary size for the tokens. While the $e$ can be any nonnegative integer, we first discuss our algorithm in the simpler case in which $e$ is a multiple  of $d$, i.e. $e = (K-1)d$, where $K \in \mathbb{N}$ and $K > 1$.
  
  \begin{figure}[ht]
  \centering
  \begin{subfigure}[b]{0.35\linewidth}
    \centering
    \includegraphics[width=\textwidth]{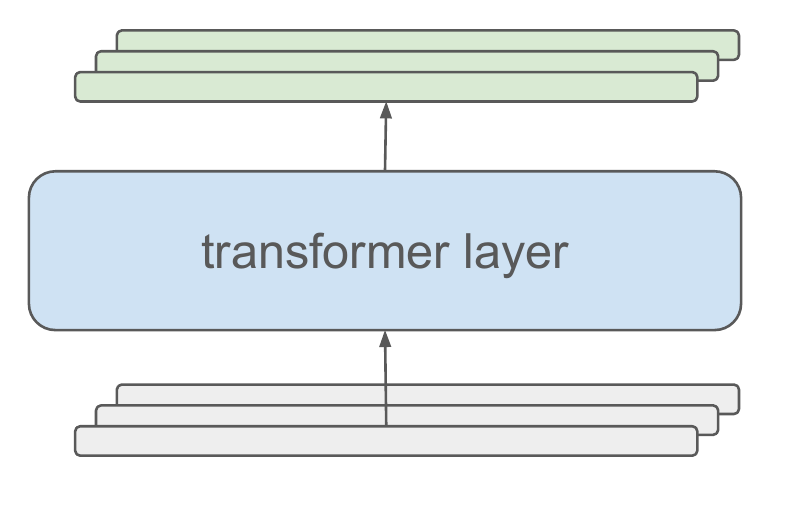}
    \caption{Wide transformer}
    \vspace{-0.05in}
    \label{fig:sparse_cube_a}
  \end{subfigure}
  \begin{subfigure}[b]{0.6\linewidth}
    \centering
    \includegraphics[width=\textwidth]{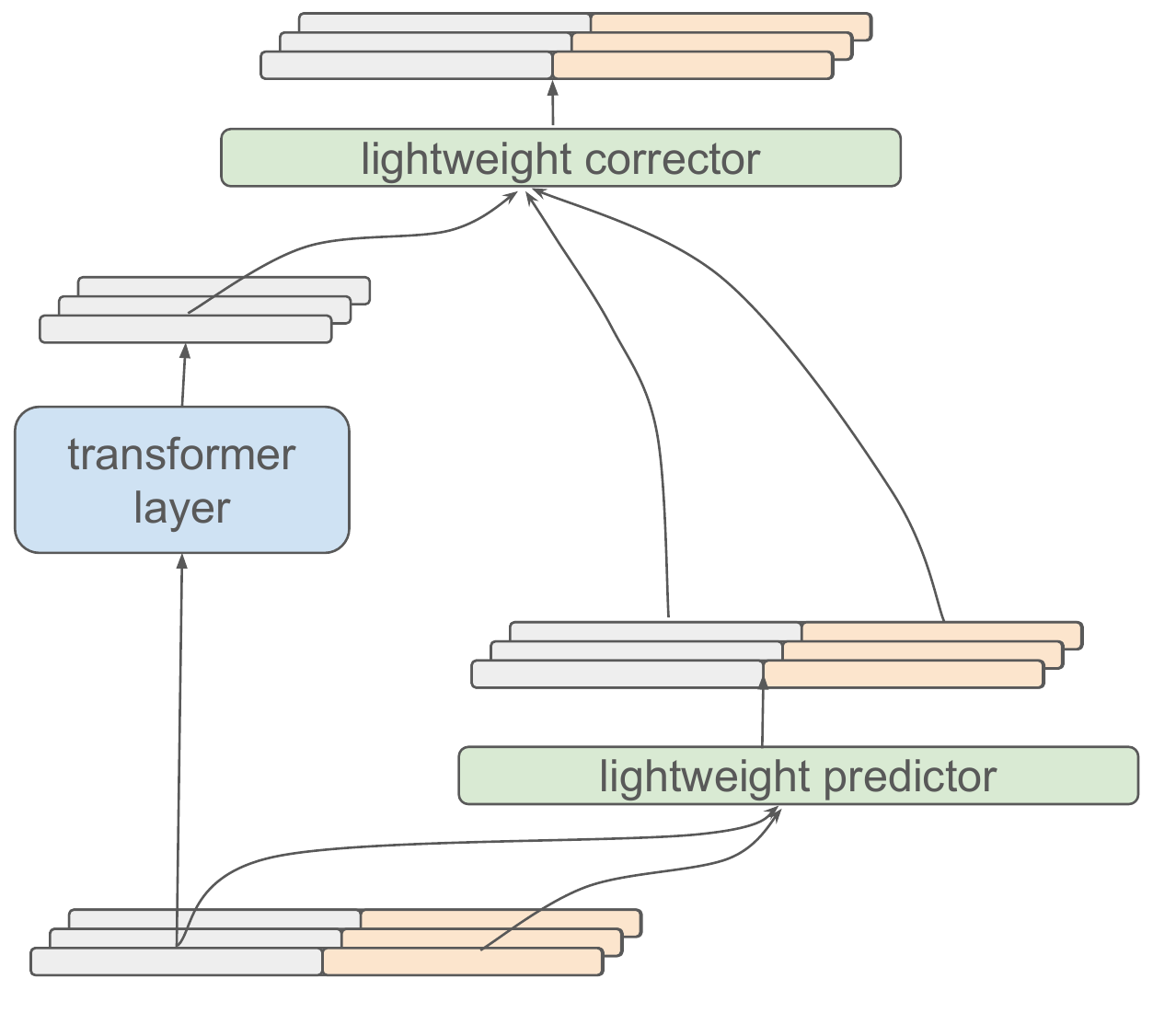}
    \caption{Predict-Compute-Correct}
    \vspace{-0.05in}
    \label{fig:sparse_cube_b}
  \end{subfigure}
  \caption{Updating a wide representation vector: (a) wide transformer layers scales quadratically with the representation dimension; (b) Predict-Compute-Correct algorithm uses a narrow transformer layer along with lightweight predictor and corrector to update a wide representation vector.}
   \label{fig:pcc}
   \vspace{-0.1in}
\end{figure}

More specifically, let the input embedding vector be $K d$ dimensional, where $K, d \in \mathbb{N}$. Our algorithm keeps the dimension of the representation vector at every layer to be $K d$ while uses layers of width $d$ to transform the representation vector. Denote the representation vector at layer $i$ by $x_i = \mathrm{concat}(x_i^1, x_i^2, ..., x_i^K)$, where $x^j_i \in \mathbb{R}^d, j = 1, 2, ..., K$ are contiguous sub-blocks of $x_i$ and $\mathrm{concat}$ is the concatenation operation. Denote layer $i$'s transformation function by $L_i: \mathbb{R}^d \rightarrow \mathbb{R}^d$. Representation vector $x_{i+1}$ at layer $i+1$ is obtained in three steps: 

\begin{enumerate}[noitemsep,nolistsep]
    \item \textbf{Prediction}: predict the representation vector at next layer with a trainable linear map $\hat{x}_{i+1} = P_i x_i$, where $P_i \in \mathbb{R}^{K d \times K d}$;
    \item \textbf{Computation}: select a sub-block $x_i^{j^*}$ and update this block with $L_i$:  $\tilde{x}_{i+1}^{j^*} = L_i (x_i^{j^*})$ (selection of $j^*$ is discussed in the next section); more than one sub-blocks can be selected if needed;
    \item \textbf{Correction}: correct the prediction with the computation result: $x_{i+1} = \hat{x}_{i+1} + G_i ( \tilde{x}_{i+1}^{j^*}  - \hat{x}_i^{j^*})$, where $G_i \in \mathbb{R}^{K d \times d}$ is a trainable matrix.
\end{enumerate}

When there is no ambiguity about the layer index, we drop the subscript $i$, and denote $x_{old} := x_i$ and $x_{new} := x_{i+1}$.  The three steps are summarized in Algorithm~\ref{alg:pcc}.

\begin{algorithm}
\SetAlgoLined
\KwIn{Representation vector $x_{old} = \mathrm{concat}(x_{old}^1, x_{old}^2, ..., x_{old}^K)$, where $x_{old}^j \in \mathbb{R}^d, j = 1, 2, ..., K$ are contiguous sub-blocks of $x_{old}$.}
\KwOut{Updated representation vector $x_{new}.$}

\textbf{Prediction}: predict the updated representation vector with a trainable linear map: $\hat{x} = P x_{old}$ , where $P \in \mathbb{R}^{K d \times K d}$ is a trainable matrix;

\textbf{Computation}: select a sub-block $x_{old}^{j^*}$ and update this block with $L$:  $\tilde{x}^{j^*} = L(x_{old}^{j^*})$;

\textbf{Correction}: correct the prediction with the computation result: $x_{new} = \hat{x} + G ( \tilde{x}^{j^*}  - \hat{x}^{j^*})$, where $G \in \mathbb{R}^{K d \times d}$ is a trainable matrix.

\caption{Predict-Compute-Correct algorithm}\label{alg:pcc}
\end{algorithm}

This Predict-Compute-Correct algorithm is inspired by the Kalman filter algorithm~\cite{kalman1960filtering}. Casted in the Kalman filtering framework, the prediction step utilizes a simple linear dynamic model,  the computation step is viewed as a form of ``measurement", and correction step performs a weighted sum of prediction and ``measurement" through the gain matrix $G$. Note the prediction and measurement noises which are important components in Kalman filter are not modeled in the above algorithm. Estimation of noises and covariance updates are left for future work.

To further reduce the computation cost of the prediction and correction steps, we impose a particular block matrix structure on $P$ and $G$. Let $P = (p_{i, j}I_{d \times d})_{i, j \in [K]}$, $G = (g_i I_{d \times d})_{i \in [K]}$,  where $p_{i, j}, g_{i} \in \mathbb{R}$ are scalars and $I_{d \times d} \in \mathbb{R}^{d \times d}$ is the identity matrix. Note this amounts to treating each sub-block $x^{i}$ as an atomic quantity. The simplified Predict-Compute-Correct algorithm is summarized in Algorithm~\ref{alg:simplified_pcc}.



\begin{algorithm}
\SetAlgoLined
\KwIn{Representation vector $x_{old} =  \mathrm{concat}(x_{old}^1, x_{old}^2, ..., x_{old}^K)$, where $x_{old}^j \in \mathbb{R}^d, j = 1, 2, ..., K$ are contiguous sub-blocks of $x_{old}$.}
\KwOut{Updated representation vector $x_{new}.$}

\textbf{Prediction}: predict the updated representation vector with a trainable linear map: $\hat{x}^i = \sum_{j = 1}^{K} p_{i, j} x_{old}^{j}$ for $i = 1, 2, ..., K$, where $p_{i, j} \in \mathbb{R}$ are trainable scalars;

\textbf{Computation}: select a sub-block $x_{old}^{j^*}$ and update this block with $L$:  $\tilde{x}^{j^*} = L(x_{old}^{j^*})$;

\textbf{Correction}: correct the prediction with the computation result: $x_{new}^{i} = \hat{x}^{i} + g_i ( \tilde{x}^{j^*}  - \hat{x}^{j^*})$ for $i = 1, 2, ..., K$, where $g_i \in \mathbb{R}$ are trainable scalars.

\caption{Simplified Predict-Compute-Correct algorithm}\label{alg:simplified_pcc}
\end{algorithm}

In the simplified algorithm, prediction and correction steps involve only vector addition and scalar-vector multiplication, which both incur $O(d)$ computation cost and much less than the $O(d^2)$ cost of the layer transformer $L$.

\subsection{Selection of sub-blocks}
The selection of sub-blocks for the computation step is not specified in Algorithm~\ref{alg:pcc} and \ref{alg:simplified_pcc}. We consider two simple, deterministic selection methods in this paper and leave more sophisticated methods for future work.

\begin{enumerate}[noitemsep,nolistsep]
    \item \textbf{Same}: choose the same sub-block for all the layers in a neural network;
    \item \textbf{Alternating}: for a sequence of layers, alternating through the sub-blocks, that is, if the sub-blocks are indexed with zero-based index, then sub-block $i$ mod $K$ is selected for the computation step for layer $i$. Algorithm \ref{alg:simplified_pcc} with alternating selection is referred to as \textbf{Alternating Updates}(\textbf{AltUp}) in the following sections.
\end{enumerate}

We compare the two selection methods empirically in Section~\ref{sec:altup_results} and found the ``alternating" method is better.

\subsection{Extension to non-integer multiples}
In the above description of the Predict-Compute-Correct algorithms, we assumed the augmented dimension $e$ is a multiple of original embedding dimension $d$. For the more general case when $e$ is not a multiple of $d$, we add a divide-and-project step before we apply Algorithm~\ref{alg:pcc} or Algorithm~\ref{alg:simplified_pcc}: choose an integer factor $(K-1)$ of $e$, divide the augmented vectors into $(K-1)$ sub-blocks, and project each sub-block to $d$ dimension. Here $K$ becomes another hyper-parameter of the algorithm.
\section{Results}
\label{sec:results}
\subsection{Setting}
We performed all of our experiments using T5-model architectures \cite{raffel2020exploring} of varying sizes (small, base, and large) which we pretrained on the C4 dataset for 500,000 \cb{500k or 100k?} steps with a batch size of $256$. The pretrained models were then finetuned on either the GLUE \cite{wang2018glue}, SuperGLUE \cite{wang2019superglue}, or SQuAD \cite{rajpurkar2016squad} benchmark tasks for a further 50,000 steps with a batch-size of $256$. The pretraining task is to predict corrupted text spans, and the finetuning tasks are re-casted into text generation tasks. We report both pretraining and finetuning metrics: for pretraining, we report span prediction accuracy on a hold-out validation set, and for finetuning, we follow the same recipe as the T5 models, see \cite{raffel2020exploring} for more details. The full experiment set-up is detailed in the appendix. 

\subsection{Memory consumption methods}\label{sec:altup_results}
In this section, we present empirical results on comparsion of different memory consumption methods, especially the Predict-Compute-Correct algorithm (Algorithm~\ref{alg:simplified_pcc}). In all the subsequent experiments, the augmented memory are implemented as additional embedding tables at the bottom layer of the model and token-ID lookup can be performed only once, which results in very small computation cost. We explore different memory consumption methods, model sizes, and memory sizes.

We first fix the augmented memory parameters to be one extra embedding table (corresonding to $K=2$ in Algorithm~\ref{alg:simplified_pcc}), lookup mechanism to be token-ID lookup, and compare different memory consumption methods. In Table~\ref{table:altup_ablation}, we compare the summation method (Sum) in which additional embedding vectors are added to the token representation vector,  Algorithm~\ref{alg:simplified_pcc} with same block selection (SameUp), and Algorithm~\ref{alg:simplified_pcc} with alternating block selection (AltUp), all on top of the T5 version 1.1 base model (B). We note all three methods bring improvements in both pretraining and finetuning, and AltUp is the most effective one. While pretraining accuracies for all three memory consumption methods are similar, differences in finetuning metrics are large, and Alternating Updates achieving roughly twice gains compared to the other two methods. Similar behaviors are observed for small and large sized T5 models, see appendix for details.

\begin{table}
    \renewcommand{\arraystretch}{1.15}
    \scalebox{0.85}{
    \begin{tabular}{@{}lcccc@{}}
    \toprule
    \multirow{2}{*}{\textbf{Model}} & {\textbf{Pretrain}} & {\textbf{Finetune}} & {\textbf{Finetune}} & {\textbf{Finetune}} \\
     & \textbf{accuracy} & \textbf{GLUE} & \textbf{SG} &\textbf{SQuAD (EM/F1)} \\
    \midrule
    B &    $66.42$ & $84.25$ & $73.56$ & $83.78/91.19$ \\
    B + Sum & $66.82$ & $84.85$ & $75.2$ & $84.36/91.36$\\
    B + SameUp &    $66.82$ & $84.06$ & $74.15$ & $84.41/91.76$ \\
    B + AltUp & $\mathbf{66.96}$ & $\mathbf{85.32}$ & $\mathbf{75.80}$ & $\mathbf{85.24/92.36}$\\
    \bottomrule
    \end{tabular}}
    \caption{Comparison of memory consumption methods: summation (Sum), Predict-Compute-Correct with ``same" block selection (SameUp), and Predict-Compute-Correct with ``alternating" block selection (AltUp).}
    \label{table:altup_ablation}
\end{table}

For the second set of experiments, we explore the alternating updates with increasing model sizes. We compare three model sizes with the T5 version 1.1 architecture: small (S), base (B) and large (L). The base and large models follow the same model configurations as in the T5 paper, while the small model is shallower than the T5 paper \cite{raffel2020exploring} to cover a larger range of model sizes ($4$ encoder/decoder layers instead of $8$ encoder/decoder layers). For models with alternating updates, we set $K=2$, corresponding to doubling the embedding dimension. Full details of the model configurations are available in Appendix. 

\begin{table}
    \renewcommand{\arraystretch}{1.15}
    \scalebox{0.85}{
    \begin{tabular}{@{}lcccc@{}}
    \toprule
    \multirow{2}{*}{\textbf{Model}} & {\textbf{Pretrain}} & {\textbf{Finetune}} & {\textbf{Finetune}} & {\textbf{Finetune}} \\
     & \textbf{accuracy} & \textbf{GLUE} & \textbf{SG} &\textbf{SQuAD (EM/F1)} \\
    \midrule
    S &    $61.21$ & $75.83$ & $59.28$ & $76.44/84.97$ \\
    S + AltUp & $\mathbf{61.86}$ & $\mathbf{76.82}$ & $\mathbf{59.60}$ & $\mathbf{77.51/85.79}$\\
    \hline
    B & $66.42$ & $84.25$ & $73.56$ & $83.78/91.19$\\
    B + AltUp & $\mathbf{66.96}$ & $\mathbf{85.32}$ & $\mathbf{75.80}$ & $\mathbf{85.24/92.36}$\\
    \hline
    L & $69.13$ & $87.23$ & $81.21$ & $86.77/93.56$ \\
    L + AltUp & $\mathbf{69.32}$ & $\mathbf{88.20}$ & $\mathbf{82.75}$ & $\mathbf{87.81/94.29}$\\
    \bottomrule
    \end{tabular}}
    \caption{T5 version 1.1. models augmented with Alternating Updates: both pretraining and finetuning metrics are improved. We observe diminishing return in pretraining accuracy gain, but no diminishing returns in finetuning metrics.}
    \label{table:altup_quality}
\end{table}

Table~\ref{table:altup_quality} shows pretraining and finetuning metrics comparsion of the baseline models and the corresponding models with Alternating Updates. Note gains in pretraining accuracies show diminishing returns when model sizes grows, gains in finetuning metrics doesn't seems to diminish. We plan to experiment with even larger models to see if this trend is still valid.

Table~\ref{table:altup_cost} documents the parameter count and training speed comparison. Note Alternating Updates increases the embedding parameters while leaving the non-embedding parameters roughly the same. Since the transformer computation are not changed by alternating updates, we also observe very small training speed impact.

\begin{table}
    \renewcommand{\arraystretch}{1.15}
    \scalebox{0.88}{
    \begin{tabular}{@{}lccc@{}}
    \toprule
    \textbf{Model} & \textbf{\# emb params} & \textbf{\# non-emb params} & \textbf{train speed} \\
    \midrule
    S &    3.29E+07	& 3.78E+07 & $166.1$   \\
    S + AltUp & 6.58E+07 & 3.99E+07 & $119.4$\\
    \hline
    B & 4.93E+07 & 1.98E+08 & $52.4$ \\
    B + AltUp & 9.87E+07 & 2.12E+08 & $42.3$ \\
    \hline
    L & 6.58E+07 & 7.17E+08 & $17.1$  \\
    L + AltUp & 1.32E+08 & 7.68E+08 & $14.4$ \\
    \bottomrule
    \end{tabular}}
    \caption{Model size and train speed comparison: T5.1.1 small(S), base(B) and large(L) models are compared. Embedding parameters include input embedding table parameters (shared between encoder and decoder) and output embedding table. Non-embedding parameters include all the transformer blocks. Train speed is measured by number of examples per second per core.}
    \label{table:altup_cost}
\end{table}

Finally, we present the model quality with different memory sizes. Table~\ref{table:altup_scaleup} contains model performances for alternating updated base sized models with $K=2$ and $4$. We observe a monotonic increasing trend for the pretraining accuracy as memory sizes increases. On finetuning tasks, we observe some tasks (SuperGLUE) continue to benefit from more memory, while other tasks (GLUE and SQuAD) don't. We hypothesize this might be attributed to the nature of the tasks: while some tasks requires more knowledge about the input and can be improved with a wider input representation, other tasks depend more function approximation capacity which is not increased with a wider input representation. We observe similar behaviors for small and large sized T5 models, see appendix for details.

\begin{table}
    \renewcommand{\arraystretch}{1.15}
    \scalebox{0.8}{
    \begin{tabular}{@{}lcccc@{}}
    \toprule
    \multirow{2}{*}{\textbf{Model}} & {\textbf{Pretrain}} & {\textbf{Finetune}} & {\textbf{Finetune}} & {\textbf{Finetune}} \\
     & \textbf{accuracy} & \textbf{GLUE} & \textbf{SG} &\textbf{SQuAD (EM/F1)} \\
    \midrule
    B &    $66.42$ & $84.25$ & $73.56$ & $83.78/91.19$ \\
    B + AltUp (K=2) & $66.96$ & $\mathbf{85.32}$ & $75.80$ & $\mathbf{85.24/92.36}$\\
    B + AltUp (K=4) & $\mathbf{67.18}$ & $84.95$ & $\mathbf{78.91}$ & $84.82/92.07$\\
    \bottomrule
    \end{tabular}}
    \caption{Performance with different memory sizes: T5 version 1.1 model augmented with larger embeddings, with $K = 2$ and $4$ corresponding to doubling and quadrupling the representation dimension.}
    \label{table:altup_scaleup}
\end{table}

\subsection{Varying table size and rank for Softmax Lookup}
Here we investigate the effect of varying table size (number of experts) and rank on the popular Softmax lookup mechanism that is used by state-of-the-art MoE models. We evaluate the performance of Softmax lookup in the partial expert setting on the performance of T5X Small pretraining with ranks $\{0, 4, 16, 64, 128\}$ and buckets (experts) $\{8, 32, 64, 128, 512\}$. The results are shown in Fig.~\ref{fig:t5x-small-softmax}. 

\begin{figure}[h!]
 \centering
 \includegraphics[width=0.5\textwidth]{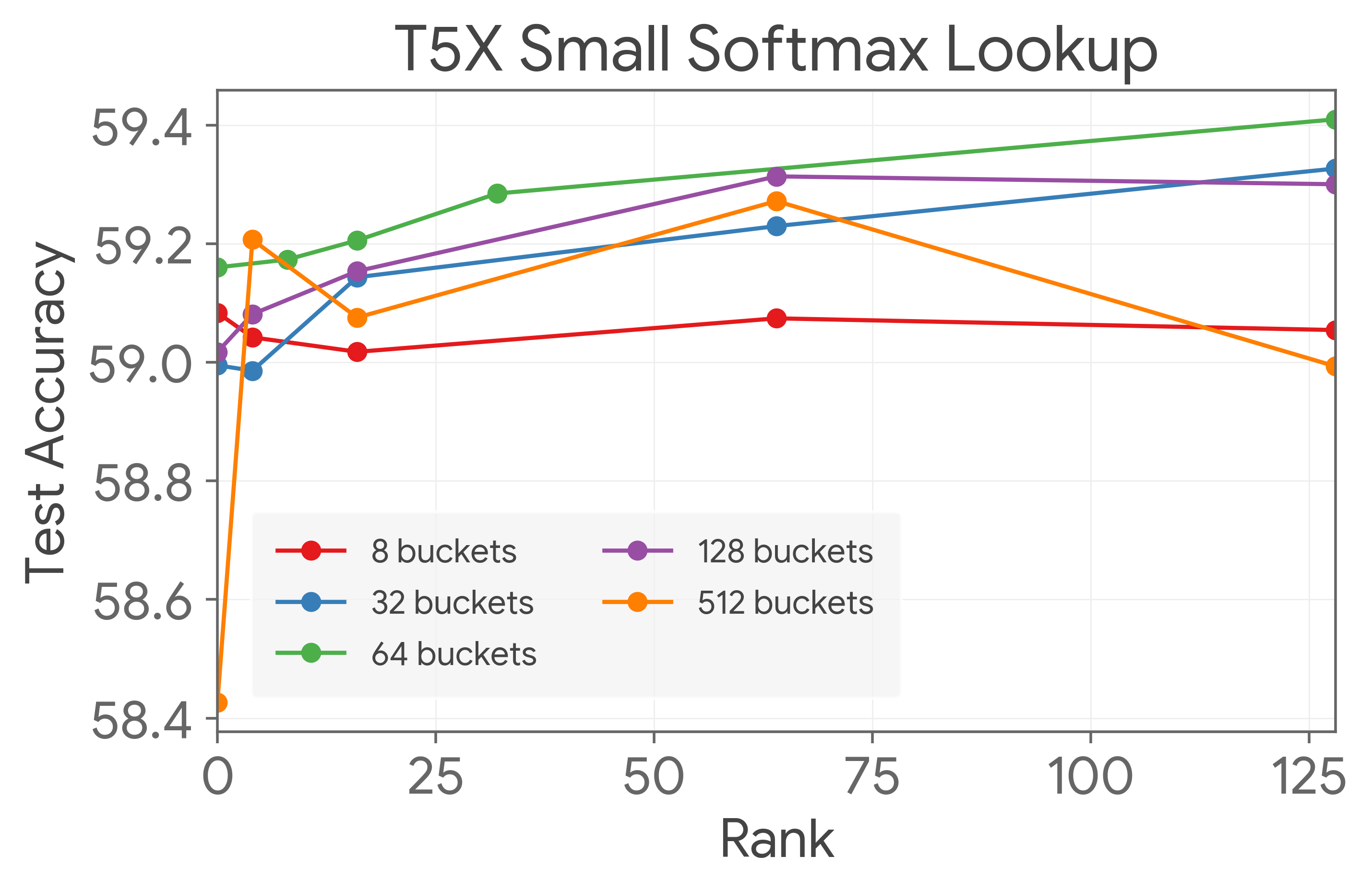}

 \caption{The performance of Softmax lookup with varying configurations of rank and buckets (number of experts) on the pretrain performance of T5 Small (baseline: $59.1\%$). Increasing the rank (width of each FF expert) generally increases performance, with rare exceptions. Increasing the number of buckets from 8 to 64 leads to monotonic increases performance, after which further increases lead to a monotonic performance degradation.}
 \label{fig:t5x-small-softmax}
\end{figure}

We observe that generally, increasing the width (rank) of the partial expert leads to increased performance for virtually all table sizes, with the notable exception of the largest table size (512). Interestingly, increasing the number of buckets (table size) does not lead to strict increases in performance. In fact, the best performing configuration uses the intermediate table size of $64$ buckets with the maximum rank tested $128$. A stronger trend also holds: increasing the buckets until 64 leads to monotonic increases in performance (see red, blue, and green curves of Fig.~\ref{fig:t5x-small-softmax}), after which we see a monotonic decrease in performance with higher number of buckets (see purple and orange curves). This is in agreement with previous observations of softmax-style MoE routing~\cite{fedus2022review}, where increasing the number of experts beyond a certain point was found to degrade performance.

\subsection{Comparisons of Lookup Functions}
We now consider evaluating the performance of the various lookup functions subject to a constraint on the number of additional parameters introduced to the model. In particular, rank 0 Token-ID introduces roughly $2^{15}$ additional parameters, so we experiment with various combinations of rank ($\{0, 2, 4, \ldots, 128, 256\}$) and the number of buckets ($\{2^3, 2^4, \ldots, 2^{15}\}$) for Softmax and LSH lookup that satisfy the criteria of adding roughly $2^{15}$ (32,128) parameters\footnote{The number of additional parameters is computed as $ \max\{2 \, \mathrm{rank}, 1\} * \mathrm{buckets}$. We refer the reader to the supplementary for details on the computation.} to the model.
 
 \begin{table}
    \renewcommand{\arraystretch}{1.15}
    \centering
    \begin{tabular}{@{}lccc@{}}
    \toprule
    \textbf{Lookup Function} & \textbf{T5 Small} & \textbf{T5 Base} \\
    \midrule
    Baseline &    59.10	& 63.35 \\
    \hline
    Token-ID (rank $0$, $\approx 2^{15}$ buckets) & \textbf{59.44} & \textbf{63.73} \\
    Softmax (rank $256$, $128$ buckets) & 59.35 & 63.49 \\
    LSH (rank $128$, $128$ buckets) & 59.12 & 63.47 \\
    \bottomrule
    \end{tabular}
    \caption{Pretrain accuracy of models with various lookup functions with the same number of additional parameters on the small and base T5 models. The value for the best-performing (rank, number of buckets) configuration subject to the constraint of adding less than $2^{15}$ parameters is reported for each lookup method. Token-ID performs the best and adds the smallest amount of additional compute relative to Softmax and LSH lookup.}
    \label{table:lookup-comparison}
\end{table}

 Table~\ref{table:lookup-comparison} depicts the results of the evaluations on T5 Small and Base models. We see that Token-ID achieves the highest pretrain accuracy on both T5 models and significantly outperforms the baselines. Softmax and LSH lookup come second and third place, respectively. Note that even though the configuration of rank 0 and $2^{15}$ buckets was evaluated for both Softmax and LSH lookups, they performed poorly in comparison to Token-ID lookup. This result precisely aligns with the statement of Theorem~\ref{thm:lookup-efficiency} which states that for a fixed size table, Token-ID is more efficient than Softmax, which is in turn more efficient than (hyperplane) LSH lookup.

\section{Conclusion}
\label{sec:conclusion}


In this paper, we study various lookup (routing) functions for sparsely activated memory modules and memory consumption methods. We empirically evaluate different lookup strategies, noting in particular the effectiveness of Token-ID lookup in the large-number-of-experts setting. We provide theoretical insights to support this experimental observation by studying the different lookup functions through the lens of Locality Sensitive Hashing. In addition, we introduce a novel method, \emph{Alternating Updates} to increase representation width with little additional computation cost. Specifically, \emph{Alternating Updates} utilizes lightweight prediction and correction steps to update a wider representation vector without increasing the transformer layer’s computation cost. As a result, we achieve strong performance improvements on language modeling and language understanding benchmarks. We envision that theoretical insights on lookup functions and the Alternating Updates algorithm can serve as valuable components for designing high-performing memory augmented models.

\bibliography{main}
\bibliographystyle{apalike}

\onecolumn


\section*{Supplementary Material for \textit{The Power of External Memory in Increasing Predictive Model Capacity}}
In this supplementary, we present the full proofs of the theoretical statements in the main paper, details of the experimental evaluations, and additional empirical results.

\section*{Proofs of the theoretical statements in Sec.~\ref{sec:analysis}}

Recall from Sec.~\ref{sec:analysis} that we consider two sentences $s_1, s_2$ of the same length $l$ that have $f$ fraction of wordpieces in common.
\lookupefficiency
\begin{proof}
For the first statement, we observe that the buckets in Spherical LSH correspond to the Voronoi regions formed by the randomly chosen set of points in $\mathbb{R}^d$. If we consider the Softmax routing matrix $W$ from Sec.~\ref{sec:routing-functions} to be random (see~\cite{andoni2015optimal} for details) and the rows to be of unit norm, then the experts' dot row-wise product with the top-1 routing will correspond to picking the expert whose routing vector is closest in angle to the input vector~\cite{andoni2015practical}. Min-hash involves hashing the universe of elements from which the sets are constructed randomly into some interval in the real line, and hashing a set to the element that has the smallest value on the real line. Thus each element can be viewed as a bucket. If we use the Jaccard simlarity measure for comparing sets (for two sets $A,B$ it is given by $\mathrm{sim}(A,B) = |A \cap B|/|A \cup B|$) then this hash function ensures that two sets $A,B$ will hash to the same bucket with probability equal to $\mathrm{sim}(A,B)$. Note that this same property holds for Token-ID lookup (Sec.~\ref{sec:routing-functions}) as each token at any layer is hashed to an expert specific to that token. In this sense, Min-hash LSH corresponds to Token-ID lookup.

For the second statement, we consider the setting of the theorem where we have two sentences $s_1,s_2$ of equal length and a fraction $f$ of wordpieces in common. Let us evaluate the probability of routing a certain token to the expert at some intermediate layer. Note that at intermediate layers we can assume that sufficient mixing has happened between the token due to self-attention modules. For simplicity, we can view this mixing as averaging the value of all the token embeddings. Since the initial wordpiece embeddings are random, after averaging the dot product between the two averages will be $f$ and so the distance between them will be $\mathcal O (\sqrt{1- f})$. On the other hand, if we take two sentences with no tokens in common the distance between them will be $\mathcal O(1)$.

An appropriate implementation of hyperplane LSH~\cite{datar2004locality} has the property $c = \mathcal O(1 / \sqrt{1 - f})$ and $\rho = \mathcal O(1/c) = \mathcal O (\sqrt{1 - f})$. With $n$ experts (buckets), this yields a collision probability of $n^{- \mathcal O(\sqrt{1 - f})}$ for the set of experts corresponding to the two similar sentences. Spherical LSH has the improved property that $\rho = \mathcal O (1 / c^2)$, which yields $ n^{-\mathcal O(1 - f)}$ for the collision probability based on the analysis above. For Min-hash, we know from the above that the probability that two sets $A$ and $B$ hash to the same bucket is equal to $\mathrm{sim}(A,B)$. This corresponds to the fraction of overlapping wordpieces in two sentences $s_1$ and $s_1$, hence the fraction of experts for which there is a collision is $f$. This proves the second statement of the theorem.

The third statement follows immediately from the second one, where the most efficient lookup is considered to be the one that has the highest probability of collision of two similar sentences. Hence, for large $n$ and small $f$, the order of efficacy is
$$
f \ge n^{- \mathcal O(1 - f)} \ge  n^{- \mathcal O(\sqrt{1 - f})}.
$$
This concludes the proof of the theorem.
\end{proof}

\embeddinglookup
\begin{proof}
We consider two architectures that implement the embedding lookup in the two distinct ways and an input $(u, q)$ with ground truth score $\dotp{\Psi(u)}{\Phi(q)}$, where $\Psi(u)$ maps $u$ to a $d$-dimensional feature vector and $\Phi(q)$ is a non-linear transformation of $q$ that can be implemented by a deep network of width $d$. The first architecture combines the lookup $\Psi(u)$ with $q$ (by a weighted sum) and feeds into the network as input; the second architecture in addition feeds the embedding output of $u$ to all the layers of the network instead of only the lowest layer. The second architecture can store $\Psi(u)$ in the table and feed it directly to the output layer (which produces $\Phi(q)$) to obtain the result $\dotp{\Psi(u)}{\Phi(q)}$ using width $d$. On the other hand, for the first architecture the entropy of the information carried up the layers is at least $2d$ assuming $u$ and $q$ are random and not correlated, and so the width of the network needs to be $2d$.
\end{proof}

\section*{Additional experiment setup details}
\label{sec:exp_setup}

\subsection*{Details of the T5 experiments}

We evaluated our techniques on the T5 language models \cite{raffel2020exploring}. Specifically, we use the T5 version 1.1 models with gated GELU feedforward network and pre layernorm. The models are implemented on top of the T5X \cite{roberts2022t5x} code base. During pretraining, we use 256 batch size, Adafactor optimizer \cite{shazeer2018adafactor} with base learning rate $1.0$ and reciprocal square-root decay with $10000$ warmup steps, and zero dropout. During finetuning, we use 256 batch size, Adafactor optimizer with constant learning rate of $0.001$ and $0.1$ dropout. Unless explicited mentioned, we pretrain for $500,000$ steps and finetune for $50,000$ steps.

\subsection*{Details of the partial expert computation}
For all experiments, we used single partial expert lookup, i.e., $k = 1$ in SMoE terminology and added the the output of the partial expert to the output of the main expert. Throughout the paper, buckets are synonymous with partial experts. In our experiments we defined each partial expert as a FF network composed of two matrices $U, V \in \mathbb{R}^{d_\mathrm{in} \times \mathrm{rank}}$, where $d_\mathrm{in}$ is the embedding dimension of the input to the partial expert and $\mathrm{rank}$ is a configurable parameter that controls the width of the expert. The output for a $d_\mathrm{in}$-dimensional input $x$ is computed as $V \phi(U^T x)$ where $\phi(\cdot)$ is the nonlinearity. In this paper, we used the ReLU function for $\phi$, i.e., $\phi(x) = \max \{0, x\}$ entrywise. Note that adding $\mathrm{buckets}$ experts, each with rank $\mathrm{rank}$ adds a total of $2 \max  \{\mathrm{rank}, 1\} * \mathrm{buckets} * d_\mathrm{in}$ parameters to the network\footnote{We ignore the $d_\mathrm{in}$ factor when comparing various routing functions since this is universally present regardless of the partial experts configuration.}. The matrices were initialized according to LeCun normal initialization~\cite{lecun2012efficient}.

\subsection*{Softmax Routing}
For softmax routing, we used the simplified implementation of the top-1 routing of~\cite{fedus2021switch}. For sake of fair comparisons with other lookup methods that do not require load balancing, we did not consider an explicit technique for load balancing such as load balancing loss~\cite{fedus2021switch} or router z loss~\cite{zoph2022designing} due to the additional hyperparameters that they introduce. We use multiplicative jitter noise sampled from a uniform distribution over $[1 - \varepsilon, 1 + \varepsilon]^{d_\mathrm{in}}$~\cite{zoph2022designing,fedus2021switch} with $\varepsilon = 0.01$. The router matrix $W$ was initialized by drawing from a zero mean Normal distribution with standard deviation $2 \times 10^{-2}$.

\section*{Additional experiment results}
\label{sec:exp_results}

\subsubsection*{Additional experiments for memory consumption methods}
We provide additional experiments for memory consumption methods comparison. In section~\ref{sec:altup_results}, we presented the comparison on the T5 version 1.1 base size model. Here we present the results for T5 version 1.1 small and large size models in Table~\ref{table:altup_ablation_sl}. We observe the Prediction-Compute-Correct algorithm with \emph{same} and \emph{alternating} block selection methods outperforms the summation method. For the small models, \emph{same} block selection method performs better in most tasks, while for large models, \emph{alternating} block selection method performs better in most tasks.

\begin{table}
    \renewcommand{\arraystretch}{1.15}
    \scalebox{0.85}{
    \begin{tabular}{@{}lcccc@{}}
    \toprule
    \multirow{2}{*}{\textbf{Model}} & {\textbf{Pretrain}} & {\textbf{Finetune}} & {\textbf{Finetune}} & {\textbf{Finetune}} \\
     & \textbf{accuracy} & \textbf{GLUE} & \textbf{SG} &\textbf{SQuAD (EM/F1)} \\
    \midrule
    S &    $61.21$ & $75.83$ & $59.52$ & $76.44/84.97$ \\
    S + Sum & $61.67$ & $77.54$ & $59.63$ & $75.06/83.82$\\
    S + SameUp &    $\mathbf{61.91}$ & $\mathbf{77.75}$ & $\mathbf{60.81}$ & $76.85/85.51$ \\
    S + AltUp & $61.86$ & $76.82$ & $59.60$ & $\mathbf{77.51/85.79}$\\
    \bottomrule
    \end{tabular}}
    \scalebox{0.85}{
    \begin{tabular}{@{}lcccc@{}}
    \toprule
    \multirow{2}{*}{\textbf{Model}} & {\textbf{Pretrain}} & {\textbf{Finetune}} & {\textbf{Finetune}} & {\textbf{Finetune}} \\
     & \textbf{accuracy} & \textbf{GLUE} & \textbf{SG} &\textbf{SQuAD (EM/F1)} \\
    \midrule
    L &    $69.13$ & $87.23$ & $81.21$ & $86.77/93.56$ \\
    L + Sum & $69.09$ & $86.18$ & $78.93$ & $86.19/93.08$\\
    L + SameUp &    $\mathbf{69.45}$ & $87.95$ & $82.72$ & $\mathbf{87.65}/ 94.13$ \\
    L + AltUp & $69.32$ & $\mathbf{88.20}$ & $\mathbf{82.75}$ & $87.58/\mathbf{94.27}$\\
    \bottomrule
    \end{tabular}}
    \caption{Comparison of memory consumption methods. T5 version 1.1 small (S) and large (L) models, with different memory consumption methods: summation (Sum), Predict-Compute-Correct with ``same" block selection (SameUp), and Predict-Compute-Correct with ``alternating" block selection (AltUp).}
    \label{table:altup_ablation_sl}
\end{table}

\subsubsection*{Additional experiments for different memory sizes}

We report the performance of T5 version 1.1 small and large models with different memory sizes in Table~\ref{table:altup_scaleup_sl}. We observe that for the large models, the trend is similar to the T5 base sized model, i.e. more memory improves both pretrain and finetune quality; while for the small model, more memory improves pretrain quality, but the finetune quality doesn't improve, likely due to overfitting from the additional parameters. 

\begin{table}
    \renewcommand{\arraystretch}{1.15}
    \scalebox{0.8}{
    \begin{tabular}{@{}lcccc@{}}
    \toprule
    \multirow{2}{*}{\textbf{Model}} & {\textbf{Pretrain}} & {\textbf{Finetune}} & {\textbf{Finetune}} & {\textbf{Finetune}} \\
     & \textbf{accuracy} & \textbf{GLUE} & \textbf{SG} &\textbf{SQuAD (EM/F1)} \\
    \midrule
    S &   $61.21$ & $75.83$ & $59.52$ & $76.44/84.97$ \\
    S + AltUp (K=2) & $61.86$ & $\mathbf{76.82}$ & $\mathbf{59.60}$ & $\mathbf{77.51/85.79}$\\
    S + AltUp (K=4) & $\mathbf{62.00}$ & $76.40$ & $59.54$ & $76.38/84.86$\\
    \bottomrule
    \end{tabular}}
    \scalebox{0.8}{
    \begin{tabular}{@{}lcccc@{}}
    \toprule
    \multirow{2}{*}{\textbf{Model}} & {\textbf{Pretrain}} & {\textbf{Finetune}} & {\textbf{Finetune}} & {\textbf{Finetune}} \\
     & \textbf{accuracy} & \textbf{GLUE} & \textbf{SG} &\textbf{SQuAD (EM/F1)} \\
    \midrule
    L &   $69.13$ & $87.23$ & $81.21$ & $86.77/93.56$  \\
    L + AltUp (K=2) & $69.32$ & $88.20$ & $82.75$ & $\mathbf{87.81/94.29}$\\
    L + AltUp (K=4) & $\mathbf{69.55}$ & $\mathbf{88.42}$ & $\mathbf{82.94}$ & $87.59/94.02$\\
    \bottomrule
    \end{tabular}}
    \caption{Performance with different memory sizes: T5 version 1.1 small (S) and large (L) models augmented with larger embeddings, with $K = 2$ and $4$ corresponding to doubling and quadrupling the representation dimension.}
    \label{table:altup_scaleup_sl}
\end{table}

\subsubsection*{Softmax Lookup Sweeps}
In this section we plot the results of additional sweeps of the $\mathrm{bucket}$ and $\mathrm{rank}$ parameters for the T5 Base and T5 Large models to supplement the sweep on T5 small (Fig.~\ref{fig:t5x-small-softmax}). The results of the sweeps are shown in Fig.~\ref{fig:softmax-base-large-comp}.

\begin{figure*}[ht!]
  \centering
  \begin{minipage}[t]{0.49\textwidth}
  \centering
 \includegraphics[width=1\textwidth]{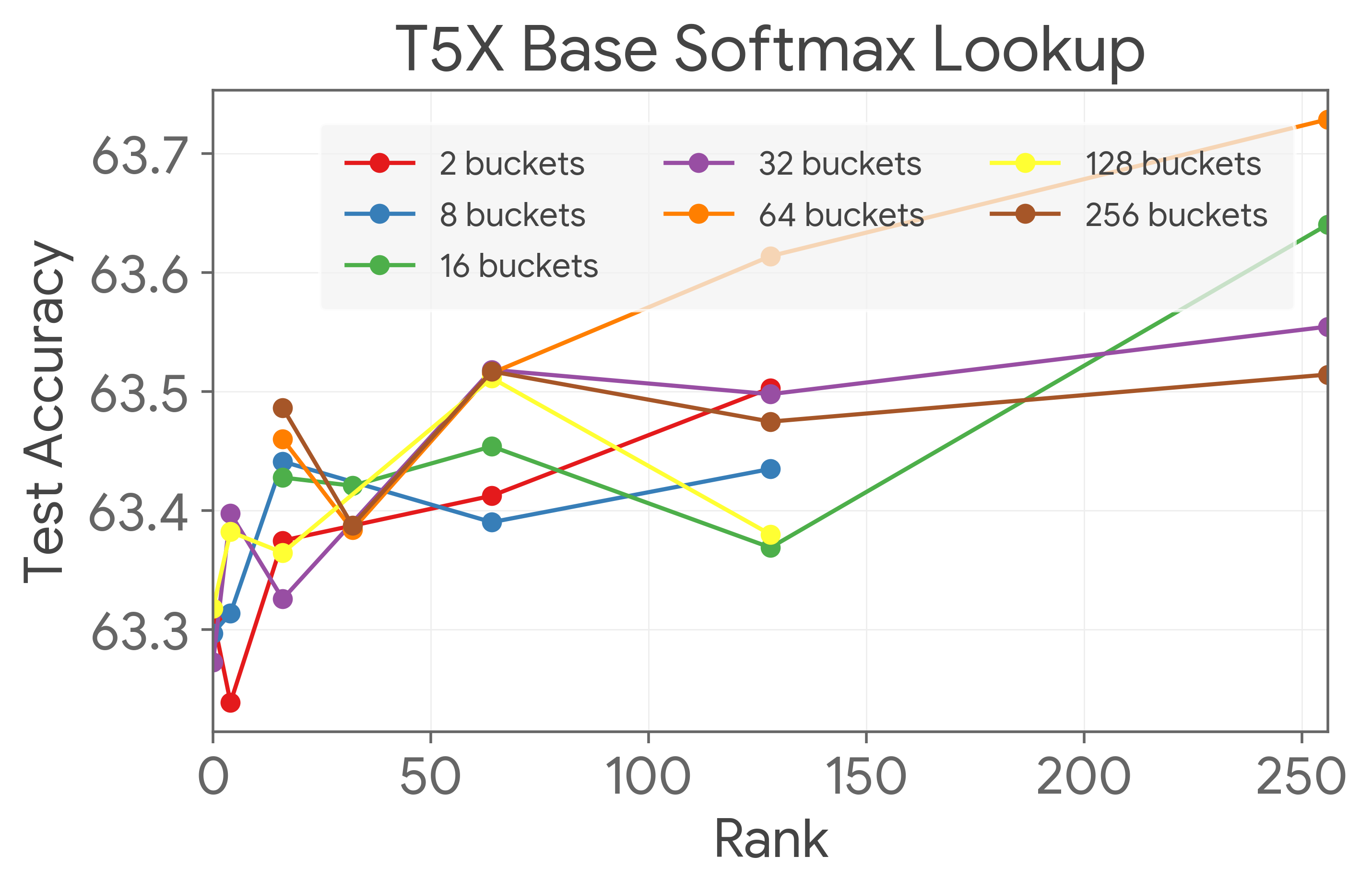}
  \end{minipage}%
  \hfill
  \begin{minipage}[t]{0.49\textwidth}
  \centering
   \includegraphics[width=1\textwidth]{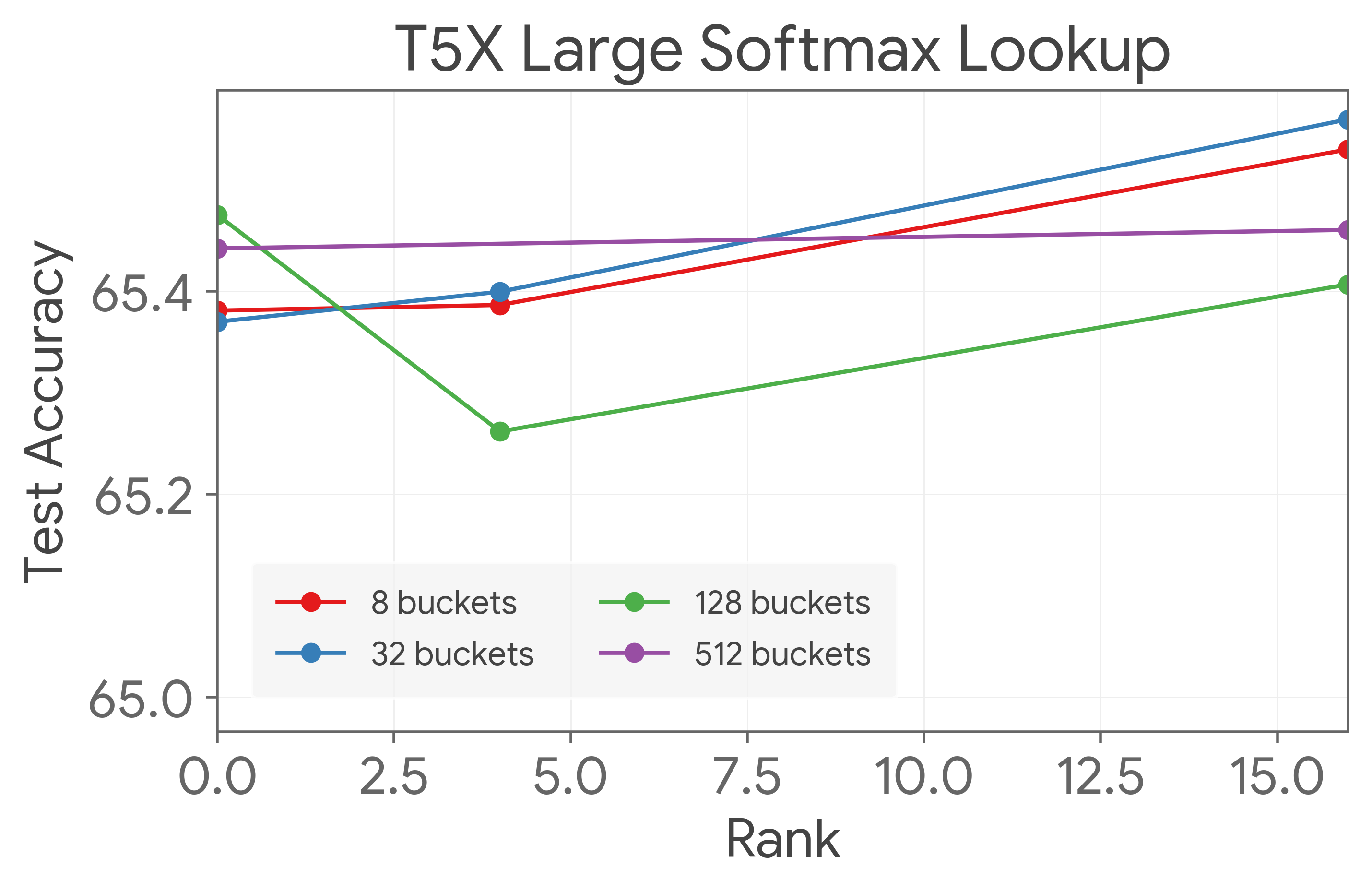}
  \end{minipage}%

\caption{The performance of Softmax lookup with varying configurations of rank and buckets (number of experts) on the pretrain performance (at 100,000 steps) of T5 Base (left, baseline: $63.35\%$) and T5 Large (right, baseline: $65.72\%$). Increasing the rank generally increases performance, with rare exceptions. Certain configurations with high rank and a large number of buckets are not shown due to memory constraints.}
\label{fig:softmax-base-large-comp}
\end{figure*}

\subsection*{Comparisons of techniques}
Table~\ref{table:lookup-comparison-detail} synthesizes the performance of the techniques presented in the paper on T5 Small, Base, and Large models. 

 \begin{table}
    \renewcommand{\arraystretch}{1.5}
    \centering
    \begin{tabular}{|l|c|c|c|c|}
    \hline
    \textbf{Technique} & \textbf{T5 Small} & \textbf{T5 Base} & \textbf{T5 Large} \\
    \hline
    Baseline &    59.10	& 63.35 & 65.58 \\
    \hline
    Alternating Updates ($K = 2$) & \textbf{59.67} & \textbf{63.97} & \textbf{66.13} \\
    \hline
    Token-ID  & \makecell{\vspace{1px} 59.44\\\small{(rank $0$, $\approx 2^{15}$ buckets)}} & \makecell{63.76\\\small{(rank $5$, $\approx 2^{15}$ buckets)}} & \makecell{\vspace{1px} 65.49 \\\small{(rank $0$, $\approx 2^{15}$ buckets)}}  \\
    \hline
    Softmax & \makecell{59.42\\\small{(rank $128$, $64$ buckets)}} & \makecell{63.62\\\small{(rank $128$, $64$ buckets)}} & \makecell{65.61\\\small{(rank $16$, $32$ buckets)}} \\
    \hline
    LSH & \makecell{59.12\\\small{(rank $128$, $128$ buckets)}} & \makecell{63.47\\\small{(rank $10$, $1024$ buckets)}} & \makecell{65.60\\ \small{(rank $10$, $1024$ buckets)}} \\
    \hline
    \end{tabular}
    \caption{Pretrain accuracy at 100,000 steps of T5 models augmented with alternating updates (see Sec.~\ref{sec:alternating-updates}) and various lookup functions with the best configuration subject to an upper bound of $\mathrm{rank} \leq 128$ and $\mathrm{buckets} \leq 2^{15}$. The value for the best-performing (rank, number of buckets) configuration is reported in the case of routing functions.}
    \label{table:lookup-comparison-detail}
\end{table}

\end{document}